\documentclass[letterpaper, 10 pt, conference]{formatting/ieeeconf}
\IEEEoverridecommandlockouts                              
\let\labelindent\relax
\usepackage{amsfonts}       

\usepackage{amsthm}
\usepackage{mathtools}      
\usepackage{amssymb}        
\usepackage{filecontents}
\usepackage[pdftex]{graphicx}
\usepackage{marginnote}     
\usepackage{marvosym}       
\usepackage{overpic}        
\usepackage{tabularx}
\usepackage{cite}
\usepackage{color}
\usepackage[linesnumbered,algoruled,boxed,lined]{algorithm2e}
\usepackage[normalem]{ulem}
\usepackage{epstopdf}
\usepackage{enumitem}
\usepackage[normalem]{ulem}
\usepackage{lipsum}
\usepackage{microtype} 
\usepackage{wrapfig}
\usepackage{comment}
\usepackage[dvipsnames]{xcolor}
\usepackage{xargs} 

\newif\ifdraft
\draftfalse

\ifdraft
\usepackage[paperheight=11in,paperwidth=9.5in,
			left=1.25in,right=1.25in,
			top=0.75in,bottom=0.75in,
			heightrounded,marginparwidth=1.2in,
			marginparsep=0.05in]{geometry}
\usepackage[textsize=footnotesize]{todonotes}
\newcommandx{\nt}[2][1=]{\todo[linecolor=red,
			backgroundcolor=red!10,bordercolor=red,#1]{#2}}
\newcommandx{\jy}[2][1=]{\todo[linecolor=green,
            backgroundcolor=green!10,bordercolor=green,#1]{JY: #2}}
\newcommandx{\sh}[2][1=]{\todo[linecolor=blue,
			backgroundcolor=blue!10,bordercolor=blue,#1]{SH: #2}}
\newcommandx{\wt}[2][1=]{\todo[linecolor=orange,
			backgroundcolor=orange!10,bordercolor=orange,#1]{SF: #2}}
\else
\newcommand{\nt}[1]{{}}
\newcommand{\sh}[1]{{}}
\newcommand{\wt}[1]{{}}
\newcommand{\jy}[1]{{}}
\fi

\newif\iftwocolumn
\twocolumntrue

\newtheorem{problem}{Problem}
\newtheorem{proposition}{Proposition}[section]

\newtheorem{theorem}{Theorem}[section]
\theoremstyle{definition}

\theoremstyle{remark}

\SetKwProg{Fn}{Function}{}{}
\SetKwComment{Comment}{$\triangleright$\ }{}

\makeatletter
\def\subsubsection{\@startsection{subsubsection}
                                 {3}
                                 {\z@ \hspace*{1mm}}
                                 {0ex plus 0.1ex minus 0.1ex}
                                 {0ex}
                                 {\normalfont\normalsize\itshape}}
\makeatother

\newcommand{\W}{\mathcal{W}}

\newcommand{\R}{\mathcal{R}}
\let\O\relax
\newcommand{\O}{\mathcal{O}}

\def\crp{{\tt {CRP}}\xspace}
\def\mrcr{{\tt {MRCR}}\xspace}

\def\tamp{{\tt {TaMP}}\xspace}
\def\namo{{\tt {NAMO}}\xspace}
\def\mpsat{{\tt {MPSAT}}\xspace}

\title{\bf
Computing High-Quality Clutter Removal Solutions for Multiple Robots
}
\author{
Wei N. Tang \quad Shuai D. Han \quad Jingjin Yu
\thanks{
W. N. Tang, S. D. Han, and J. Yu are with the Department of 
Computer Science, Rutgers, the State University of New Jersey, 
Piscataway, NJ, USA. E-Mails: 
\{{\tt wei.tang,shuai.han,jingjin.yu}\}\hspace*{.25em}
\MVAt \hspace*{.25em}rutgers.edu. 
This work is supported by NSF awards IIS-1734419 and IIS-1845888.
}%
}

\begin{document}

\maketitle
\thispagestyle{empty}
\pagestyle{empty}


\begin{abstract}
We investigate the task and motion planning problem of clearing 
clutter from a workspace with limited ingress/egress access for 
multiple robots. We call the problem multi-robot clutter removal 
(\mrcr). Targeting practical applications where motion planning 
is non-trivial but is {\em not} a bottleneck, we focus on finding
high-quality solutions for feasible \mrcr instances, which depends 
on the ability to efficiently compute high-quality object removal 
sequences. Despite the challenging multi-robot setting, our 
proposed search algorithms based on A${}^*$, dynamic programming, 
and best-first heuristics all produce solutions for tens of objects 
that significantly outperform single robot solutions. 
Realistic simulations with multiple Kuka youBots further confirms 
the effectiveness of our algorithmic solutions. 
In contrast, we also show that deciding the optimal object removal 
sequence for \mrcr is computationally intractable. 
\end{abstract}

\section{Introduction}\label{sec:intro}
%
This study expands the investigation of the {\em clutter removal problem} 
(\crp) \cite{TanYu2019ISRR} to the case where multiple robots are 
available. Specifically, we target the setting where several robots 
operate in a constrained workspace where an exit is shared, 
and the task is to remove objects that are initially scattered in the workspace. 
We call this the {\em multi-robot clutter removal} (\mrcr) problem 
(see Fig.~\ref{fig:intro} for an example). 
The shift from a single robot to multiple robots brings two key
challenges. 
First, the robots must share the free space, especially when they are 
close to the exit, which negatively impacts the computational efficiency 
when high-quality plans are sought after. 
Second, intricate interactions arise when there are more than one robot. 
Consider the scenario in which one robot is scheduled to grasp an object 
(say $o_1$), while $o_1$ is currently blocking the access to another object $o_2$. As we plan 
for a second robot, an optimal plan must account for the possibility that 
$o_2$ becomes available after the first robot picks up $o_1$, even 
though $o_2$ is not accessible at the moment.  
%

\begin{figure}[ht]
    \centering
    \includegraphics[width = 0.99\columnwidth]{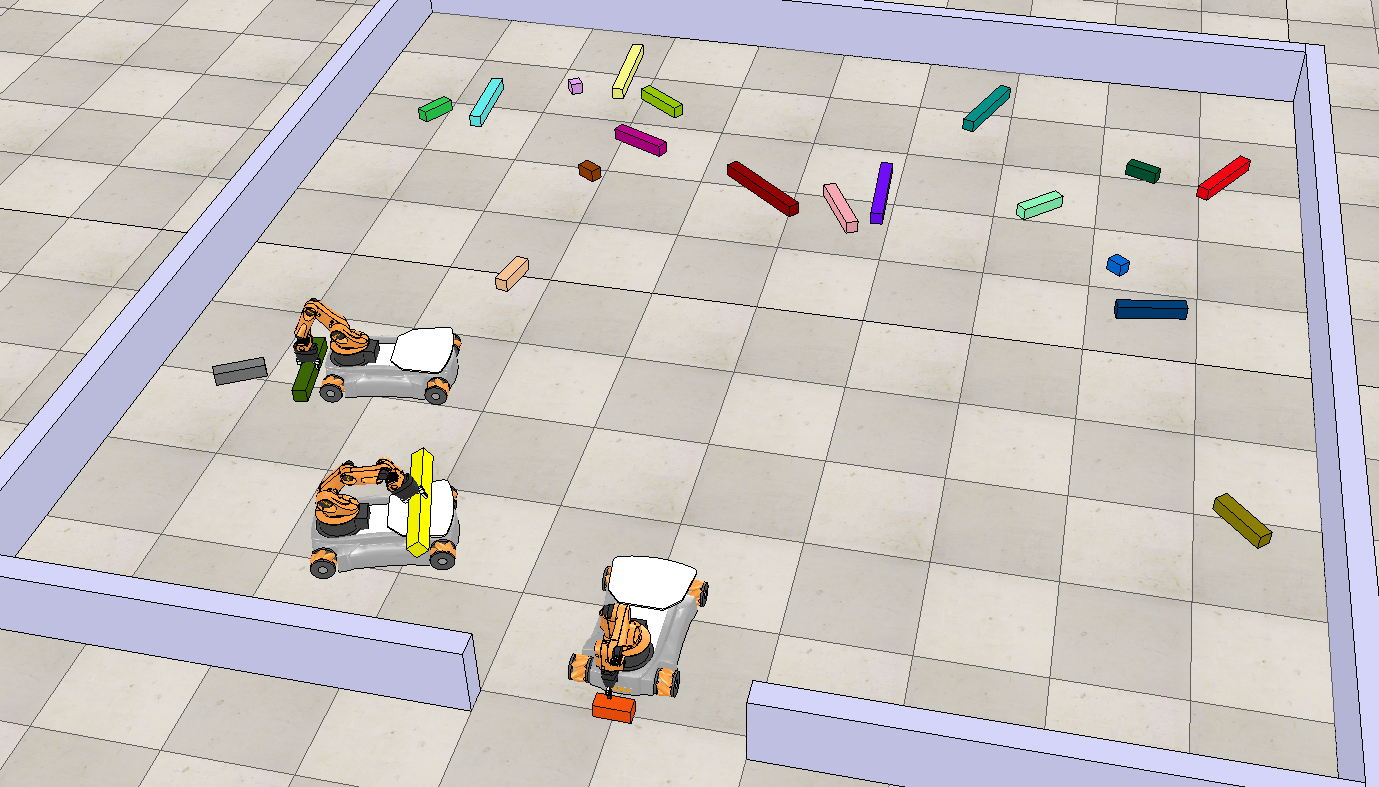}
    \caption{\label{fig:intro}
    A snapshot from a simulation run with three Kuka youBots solving 
        a clutter removal task. At the moment, the robot at the bottom is 
        placing an orange object at the drop-off location (i.e., the exit), 
				the robot in the middle is carrying a yellow object to the exit, 
        while the robot at the top is retrieving a dark green object. 
    To solve the problem efficiently, careful coordination among the 
		robots is necessary. 
    }
\end{figure}

To address these challenges, in this paper, multiple best-first and 
near-optimal algorithmic solutions are developed for the computation 
of high-quality object removal sequences for multiple robots. These 
solutions are empirically shown to be of high quality when they are 
compared with linearly scaled single-robot solutions. For example, 
for $15$ objects, solutions computed using dynamic programming use 
only less than $60\%$ the execution time required for a single robot.
This closely match the theoretical lower bound of $50\%$, which is 
clearly impossible to achieve due to inevitable robot-robot 
interaction. In contrast to the single-robot case \cite{TanYu2019ISRR}, 
for typical multi-robot settings, though greedy best-first methods work reasonably 
well and run faster, they fall significantly behind more optimal 
methods in terms of solution quality. V-REP~\cite{RohSin2013IROS} based 
simulation further confirms the removal sequences computed by our proposed 
methods remain effective when they are integrated into complete solutions. 

From the practical point of view, our study is motivated by the need
of deploying mobile robot systems for disaster response tasks
\cite{pratt2013darpa,murphy2014disaster}. Clearly, effective and 
autonomous disaster response requires the close integration of 
robotics hardware and advanced algorithmic solutions spanning computer 
vision (e.g., scene understanding), planning, among others. 
Our work focus at calculating the optimal object picking sequence, 
which is often called {\em task planning} in a {\em task and motion planning} (\tamp) 
problem \cite{cambon2009hybrid,plaku2010sampling,srivastava2014combined,
garrett2015ffrob,dantam2016incremental}. 

Generally, solving a \tamp challenge requires discrete combinatorial 
reasoning and (continuous) motion planning, both of which can often 
be computationally hard \cite{wilfong1991motion,chen1991practical,
demaine2000pushpush,nieuwenhuisen2008effective,hopcroft1984complexity,
kavraki1993complexity}. Nevertheless, effective algorithmic solutions 
have been proposed for solving many practical settings. A problem 
bearing similar combinatorial challenge as \mrcr is the problem of 
{\em Navigation among Movable Obstacles} (\namo). When the problem 
instance has {\em monotone} property, i.e., a solution exists that 
requires moving each obstacle once, backtracking techniques may be 
applied to effectively solve it \cite{stilman2008planning}. 
Probabilistic complete solutions for non-monotone settings have also 
been proposed \cite{van2009path}. In solving the single robot \crp,
we further note that {\em dynamic programming} can be integrated 
into the backtracking process \cite{TanYu2019ISRR} to reduce the 
search complexity significantly, from $O(n!)$ to $O(n2^n)$. 

Object rearrangement, e.g., \cite{ota2009rearrangement,krontiris2015dealing,
HanStiKroBekYu18IJRR,HanFenYu20ICRA-RAL} is a class of \tamp problems closely 
related to \mrcr. Most of these studies focus on the effective generation of 
a rearrangement sequence to reach a desired spatial order of 
multiple objects. Some formulations \cite{ota2009rearrangement,havur2014geometric}
in this domain are much like \namo. Whereas a search based approach is used 
in \cite{ota2009rearrangement}, symbolic reasoning is applied in 
\cite{havur2014geometric}. In contrast, \cite{krontiris2015dealing,
HanStiKroBekYu18IJRR} put more emphasis on taming the combinatorial explosion 
caused by the multiple objects involved. As it turns out, the combinatorial
aspect is highly non-trivial. For example, rearranging unlabeled objects 
is already NP-hard if an optimal solution is sought after 
\cite{HanStiKroBekYu18IJRR}. 

Our study builds on efforts aimed at developing integrated \tamp 
solutions \cite{cambon2009hybrid,plaku2010sampling,srivastava2014combined,
garrett2015ffrob,dantam2016incremental}. However, the current work 
distinguish itself in that it attacks {\em optimality issues} under the 
\crp formulation but for multiple robots. In this aspect, the focus is 
similar to \cite{vega2016asymptotically,HanStiKroBekYu18IJRR,TanYu2019ISRR}.
We note that this work does not consider other equally important aspects 
in \tamp including grasp planning \cite{ciocarlie2009hand,bohg2014data}, 
high-fidelity motion planning of robot arm \cite{simeon2004manipulation,
berenson2011task,zucker2013chomp,cohen2014single}, non-prehensile 
manipulation \cite{cosgun2011push,dogar2011framework,agboh2018pushing}, 
or uncertainty rising from perception and motion 
\cite{chang2012interactive,van2014probabilistic,eitel2017learning,
mahler2017dex,moll2018randomized}.

\textbf{Main Constributions.} First, we develop several 
practical combinatorial algorithms that generate high-quality object 
removal sequences for multiple robots (the problem has a search space
of size $O(n!k^n)$ where $n$ is the number of objects and $k$ is number
of robots). These algorithms include novel A${}^*$ and dynamic programming 
variants which produce solutions approaching theoretical optimality 
limits. Second, through realistic simulations using multiple Kuka 
youBots, we verify that our algorithms remain effective under practical 
application setups, providing significant 
savings in execution time as compared with single robot solutions. 
This validates the premise that a high-quality object removal 
sequence is a main performance impacting factor in real-world clutter
removal tasks. As a minor contribution, we also show that \mrcr is 
NP-hard to optimally solve.

The rest of the manuscript is organized as follows. In 
Section~\ref{sec:problem}, we define \mrcr and provide an overview of 
our \mrcr solution pipeline. 
In Section~\ref{sec:pipeline}, we describe \mrcr's 
structural properties, including NP-hardness if one attempts to optimally 
solve a part of it, and an intricate dependency between two 
robots and two objects. In Section~\ref{sec:algorithm}, we provide 
combinatorial algorithms for deciding high-quality object removal
sequences that also match objects with robots. We present 
experimental results in Section~\ref{sec:evaluation} and conclude 
in Section~\ref{sec:conclusion}.

\section{Preliminaries}\label{sec:problem}
\subsection{The Multi-Robot Clutter Removal Problem}
We consider the setting in which $k \ge 2$ mobile robots $\R = \{r_1, \dots, r_k\}$ 
are to clear $n$ rigid objects $\O = \{o_1, \ldots, o_n\}$ {\em scattered} on the 
ground, i.e., the objects are isolated from each other. 
We denote the workspace as $\W \subset \mathbb R^2$, and its 
boundary as $\partial \W$. 
The robots are initially placed outside of $\W$, and can enter $\W$ 
through an {\em exit} on $\partial \W$. Each robot is capable of grasping and 
transporting a single object at a time. Each object may be picked up once 
and must be subsequently transported outside of $\W$. An object is considered 
{\em cleared} after it is carried by the robot outside the exit and dropped off. 
The problem studied in this paper is defined as below. 

\begin{problem}\textbf{Multi-Robot Clutter Removal (\mrcr). }
    Given $\W, \R, \O$, find a time-optimal plan to clear all objects in $\O$.
\end{problem}

A typical \mrcr instance with $k = 3$ is illustrated in Fig.~\ref{fig:intro}. 
In this study, we assume the given \mrcr instance is always feasible and focus
on optimizing the solution to minimize the task completion time, also known as
the {\em makespan}.

\subsection{Task and Motion Planning Pipeline}
As \tamp is generally computationally intractable, approximation is 
necessary. We apply a hierarchical approach where a task planner takes charge 
of the overall planning process. 
As shown in Fig.~\ref{fig:pipeline}, a general solution pipeline for \mrcr 
starts with the acquisition of environment setup including object 
poses, workspace boundaries, and the ingress/egress location. In the 
current work, we assume such information is given or can be directly 
retrieved from the simulator (V-REP) backend.  
In a real-world setting, which we plan to tackle in future work, this 
step would require perception techniques including object detection,
scene understanding, and pose estimation, among others. 
\begin{figure}[ht!]
    \includegraphics[trim = {0 0 155 75}, clip, width=0.99\columnwidth]{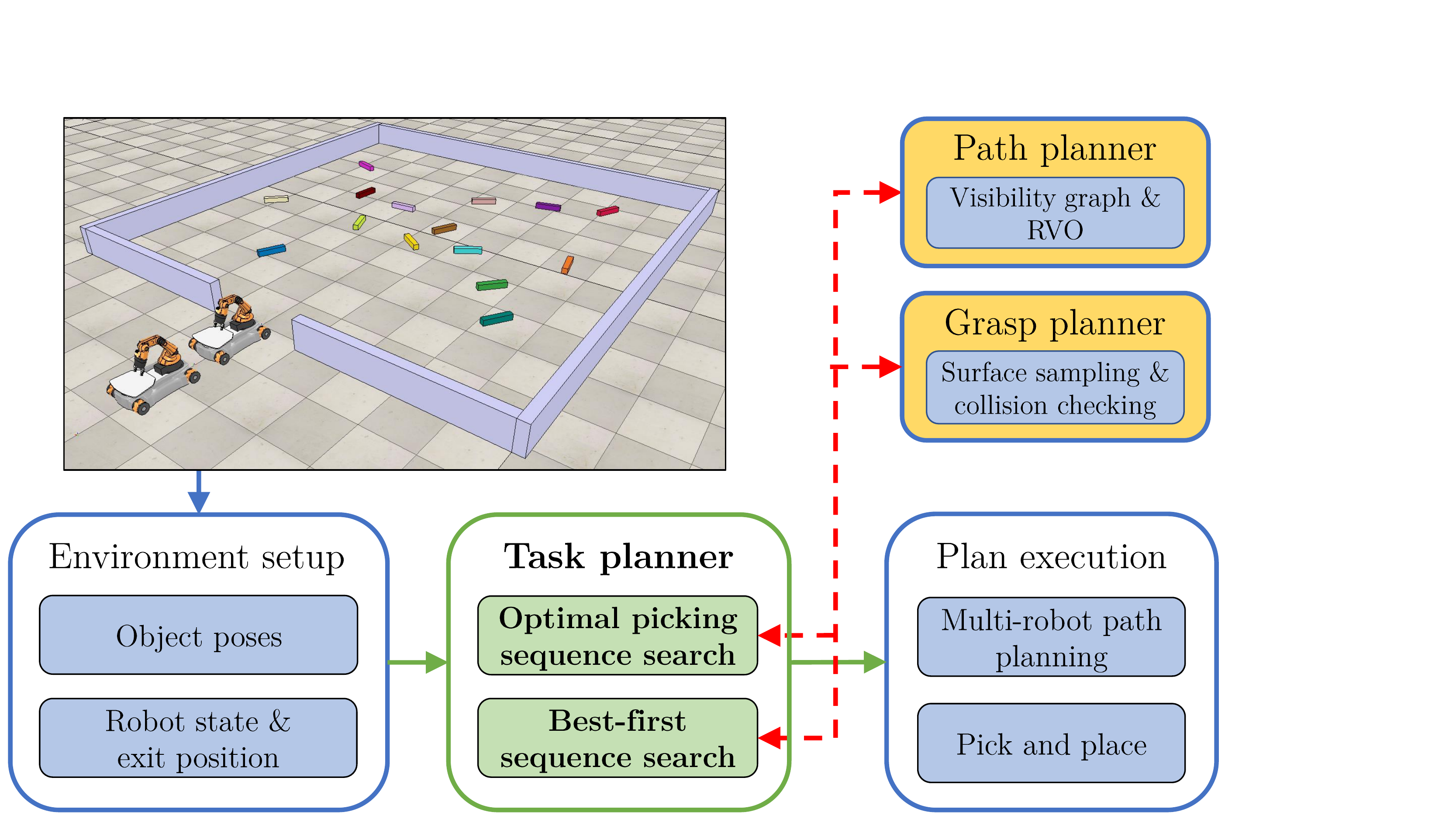}
    \caption{\label{fig:pipeline}
    The full clutter removal pipeline. 
    }
\end{figure}

With the necessary information acquired about the problem setup, the 
{\em task planner} is called next, which handles the task of matching 
each robot to an ordered sequence of objects to be removed. In general, $k \ll n$ 
so each robot will be matched with multiple objects. Multiple algorithms
are proposed (detailed in Section~\ref{sec:algorithm}) that trade
off between solution optimality and computational efficiency; a 
specific choice can be decided according to an application. During the 
task planning step, a {\em grasp planner} is used to find grasping configurations, 
and a {\em multi-robot motion planner} is used to evaluate the time 
cost for traveling to a certain target object to be removed. 
By integrating the grasp planner and the multi-robot motion planner 
into the task planner, we can compute a desirable picking 
sequence, the corresponding reference robot trajectories for reaching 
the target objects, and object grasping plan, all at the same time.

As our work does not focus on non-prehensile manipulation, we work 
with cuboid-like objects. For such objects, 
\begin{wrapfigure}[6]{r}{0.9in}
    \includegraphics[width=0.88in]{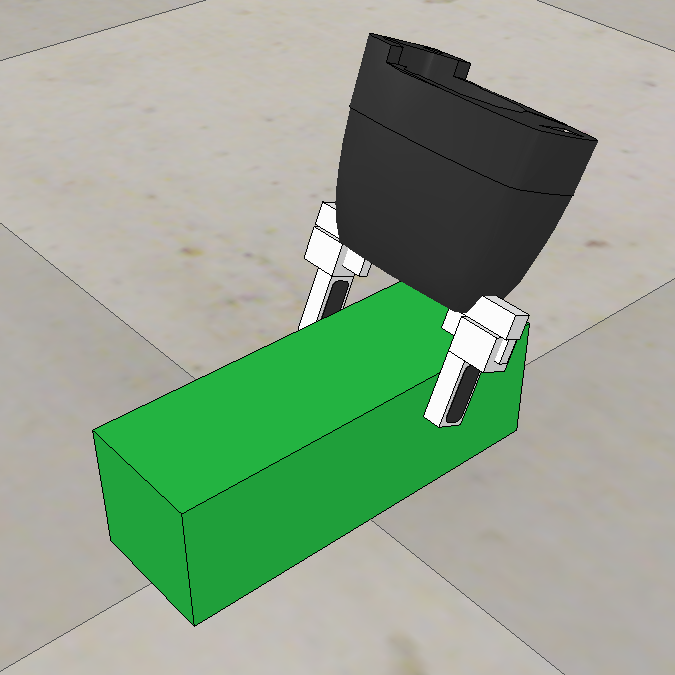}
\end{wrapfigure} 
we apply a relatively 
simple grasp planner: for each object, the planner first 
finds the top face (i.e., the one with surface normals pointing up) 
and samples the normals for possible grasps by a 2-finger gripper. 
As an illustration, in the figure on the right, the simple 
$2$-finger gripper is placed at a sampled grasp pose over an 
accessible object identified by the planner. Generally, at each 
phase, many such grasping poses are sampled. 

Given the robots' current configurations and the target objects' 
poses, the task for the multi-robot motion planner is to plan 
time-efficient, collision-free trajectories for the robots to retrieve 
the objects from the workspace. Unlike \crp 
where only a single robot appears in the scene, in \mrcr, we must 
avoid robot-robot collisions while ensuring that the resulting paths 
have short makespan. In this work, we experimented with 
both a dRRT${}^*$ \cite{shome2019drrt} based planner and a planner 
that combines visibility graph (VG)~\cite{lozano1979algorithm} and 
Reciprocal Velocity Obstacles (RVO)~\cite{van2011reciprocal}. 
dRRT* worked for two robots while VG-RVO worked well for two or 
more robots. In the two-robot case, we did not observe significant 
optimality difference between the dRRT* solution and the VG-RVO 
solution. As such, we only provide evaluation results on VG-RVO. 
Other multi-robot motion planners can also be used. 


We note that the pipeline can be 
made resolution complete and asymptotically optimal for solving 
\mrcr, e.g., by using an optimal sampling-based motion planner 
within the task planner. However, the approach is unlikely to 
be scalable. 

\subsection{Special Cases and Problem Extensions}
In the study of the single robot version of \mrcr \cite{TanYu2019ISRR} 
(i.e., $k = 1$), we have explored additional cases including multiple exits, 
static obstacles within the workspace, and different object placements. 
From that study, we have observed that internal static obstacles do not 
adversely affect algorithm running time holding object density unchanged. 
Also, cases with multiple exits and axis-aligned or overlapping objects 
can be handled with proper techniques and are generally simpler to solve. 
In the multi-robot scenario, we admit that these variations and extensions 
may make the motion planning part harder to solve, but they do not add much more 
complexity to the task planner. Therefore, in 
our study of \mrcr, we focus on the most challenging single-exit case without 
considering static internal workspace obstacles, and mainly consider cases 
where objects are randomly placed without overlapping. 


\section{Structure and Hardness of \mrcr}\label{sec:pipeline}
\subsection{Unique Structural Properties of \mrcr}\label{subsec:structure}
For an \mrcr instance with $k$ robots and $n$ objects, in the worst case, 
there are $n!k^n$ possible assignments of robots to objects over the time 
horizon: there are $n!$ possible sequences with which the $n$ objects may 
be picked up over time; for each such sequence of length $n$, 
there are $k^n$ possible ways of assigning the $k$ robots. 

Beside the combinatorial explosion induced by multiple robots, there are 
two more differences that distinguish \mrcr from single robot \crp. First, with a 
single exit of limited width, robots must non-trivially coordinate their 
movement to avoid collisions. Second, there can be more intricate 
dependencies between multiple robots and objects that impact solution 
optimality. Fig.~\ref{fig:rordep} illustrates a simple case that may 
happen between two robots and two objects. Essentially, when there are 
multiple robots, {\em lookahead} (i.e., simulating execution of plan into the
future) is necessary to optimize a plan. Our algorithmic solutions 
(in Section~\ref{sec:algorithm}) are tailored to address these unique 
complications induced by \mrcr. 

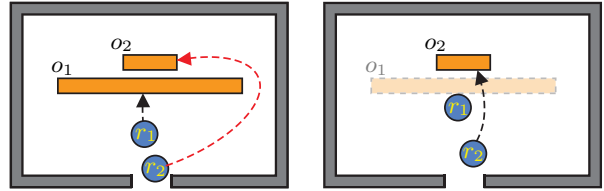
\begin{figure}[ht!]
\vspace*{3mm}
    \small
    \centering
    \begin{overpic}[width=3.1in,tics=5]{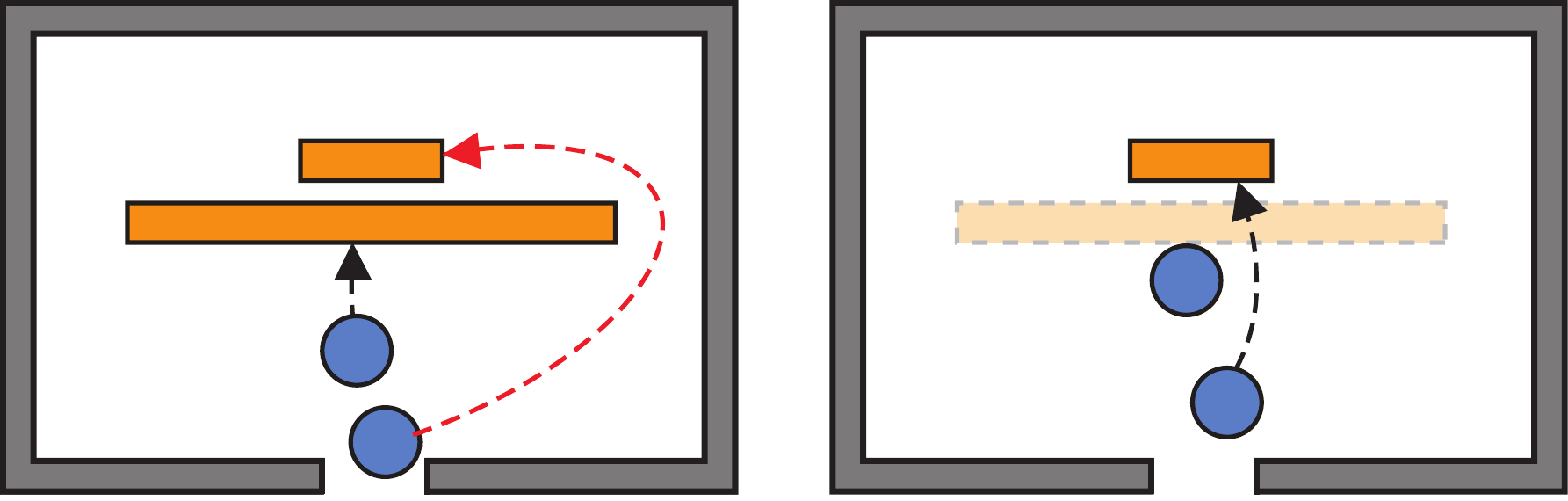}
        \put(7,20){$o_1$}
        \put(60,20){\textcolor{gray}{$o_1$}}
        \put(17,24){$o_2$}
        \put(70,24){$o_2$}
        \put(21.2,8.75){\textcolor{yellow}{$r_1$}}
        \put(22.8,3){\textcolor{yellow}{$r_2$}}
        \put(74,13.2){\textcolor{yellow}{$r_1$}}
        \put(76.5,5.5){\textcolor{yellow}{$r_2$}}
    \end{overpic}
    \caption{\label{fig:rordep} An example illustrating an intricate dependency 
    that may happen among two objects $o_1, o_2$ and two robots $r_1, r_2$. 
    [left] If we do not look into the future, $r_2$ will need to take 
    a long detour to pick up $o_2$. 
    [right] However, it is clear that once $r_1$ lifts $o_1$, 
    $r_2$ may readily access $o_2$.}
\end{figure} 

To conclude this subsection, we mention that for a given environment 
and object set, if a single robot requires at least time $T$ to complete 
the object removal task, then for $k$ robots, the minimum possible makespan 
is no less than $\frac{T}{k}$, i.e., 
\begin{proposition}\label{pro:linearspdup}
For an \mrcr instance with $k$ robots, if the 
corresponding \crp task for a single robot can be optimally solved 
in time $T$, then $k$ robots require at least $\frac{T}{k}$ makespan.  
\end{proposition}

Clearly, $\frac{T}{k}$ is a theoretical limit that is 
generally not achievable due to interactions among robots and objects. 
We will be comparing our algorithmic solutions to this limit. 

\subsection{Hardness of Selecting Optimal Object Removal Sequence}
In an extended version \cite{TanYu2019ISRREXT} of \cite{TanYu2019ISRR}, 
\crp is shown to be NP-hard to optimally solve when there is a single 
exit (the conference version \cite{TanYu2019ISRR} only provided hardness 
proof for multiple exits). In the proof, it was shown that a planar 
arrangement of objects (Fig.~\ref{fig:crpred}) can be made for which 
the optimal picking sequence is NP-hard to compute. The structure is 
constructed from a {\em monotone planar 3-SAT} (\mpsat) instance 
\cite{de2010optimal}. Based on this result, we may establish 
that the combinatorial part of \mrcr is also hard to solve. 
\begin{figure}[ht!]
\vspace*{1mm}
\begin{center}
\begin{overpic}[width=1.5in,tics=5]{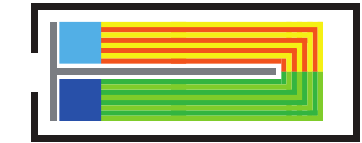}
\end{overpic}
\end{center}
\vspace*{-1mm}
\caption{\label{fig:crpred} A sketched illustration of the object arrangement
used in proving that \crp is NP-hard to solve optimally.}
\end{figure} 

\begin{theorem}
Deciding an optimal object removal sequence is NP-hard for \mrcr. 
\end{theorem}
\begin{proof}
The reduction from \mpsat to optimal \crp is fairly complex 
\cite{TanYu2019ISRREXT}. However, for this proof, we only need the fact 
that the colored pieces can be only removed in a mostly sequential manner.
Moreover, an optimal removal sequence translates to a solution to the 
original \mpsat instance. For our reduction from \mpsat to \mrcr with 
$k$ robots, we essentially make $k$ copies of the structure. For two 
robots, a possible reduced instance is given in Fig.~\ref{fig:mrcrred}. 
\begin{figure}[ht!]
\vspace*{1mm}
\begin{center}
\begin{overpic}[width=3.2in,tics=5]{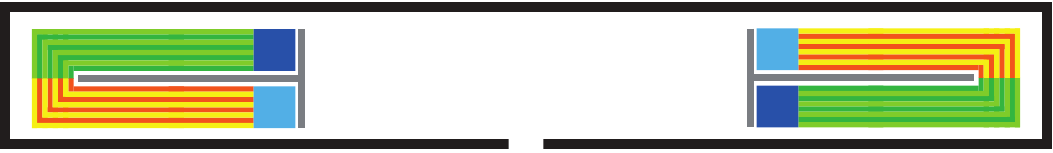}
\end{overpic}
\end{center}
\vspace*{-1mm}
\caption{\label{fig:mrcrred} A sketched illustration of the object arrangement
used in proving that \mrcr is NP-hard to solve optimally.}
\end{figure} 

In the figure, two of the reduced structures are placed on the left and right 
ends of the environment and there is a single exit in the middle. If this setup is 
to be solved with a single robot, it can be partitioned into a left task 
and a right task, each of which take the same amount of minimum time to 
complete, say $T$. The minimum makespan for the full problem is then $2T$ which is 
NP-hard to compute by \cite{TanYu2019ISRREXT}. If we use two robots, then 
the problem can be solved with a makespan of $T$ by asking each robot to take 
care of one side. On the other hand, given the sequential nature in 
solving the individual side, having two robots working on the same 
side (e.g., the left one) cannot reduce the required makespan. As 
such, the \mrcr problem is NP-hard to solve as well. The generalization to 
arbitrary $k$ is straightforward. 
\end{proof}

\section{Computing High-Quality Object Removal Sequence and Robot Assignments}\label{sec:algorithm}
As mentioned, unlike the single robot setting~\cite{TanYu2019ISRR}, 
\mrcr has a unique, more complex structure that makes it much more 
challenging computationally. To begin to address the challenge, we 
first provide a description of the system state for tracking progress 
in a search algorithm. 

\textbf{Problem State Discretization.} 
While there are many possible ways of doing 
this, we discretize the continuous problem at time instances when 
there is at least one robot at the exit (the drop-off location) and 
ready to start retrieving an object. 
With this definition, a state is essentially a crucial time point 
when a decision must be made in terms of which objects the free 
robots are going to retrieve next.
We denote a state as $\{R, 
O_1, O_2, \dots, O_k, O_{\text{NA}}, \mathcal S\}$ with $R$ being 
the set of robots currently idle, $O_j$ ($1 \leq j \leq k$) 
being the sequence of objects that are already assigned to 
robot $r_j$ and are (or being) retrieved by $r_j$. 
$O_{\text{NA}} \coloneqq \O \setminus \bigcup_{1 \leq j \leq k} 
O_j$ is the set of objects not yet assigned to any robot. 
$\mathcal S$ contains the low-level motion planning elements. 

\textbf{A Basic Search Process.}
Beginning from the start state where $O_{\text{NA}} = \O$, 
$R = \R$ and all robots are at the entrance, a discrete 
search can be carried out to find a picking sequence. To proceed 
to a next state, we assign each available robot in $R$ an 
object from $O_{\text{NA}}$. Here, the maximum number of possible assignments 
is $|R|! \binom{|O_{\text{NA}}|}{|R|}$. 
For each possible assignment, a low-level grasp and motion planning 
problem is tackled, which involves two steps: 
{\em (i)} for each object assigned to retrieve, the grasp planner is called 
to provide a feasible configuration for the designated robot to grasp 
the object; 
{\em (ii)} the multi-robot path planner is called to plan the paths for the 
robots to move their respective grasping configuration, grasp the objects, 
and then move back to the entrance to drop-off. 
At the end of the low-level search, a robot drops off an object and 
is ready to start to retrieve another object. Now, the search process reaches a new 
state by our definition.  
This searching process reaches a goal state 
when all the objects are retrieved from the workspace. 

An illustration of a search tree section is provided in Fig.~\ref{fig:search}. 
At the top left state, robot $r_2$ is on its way retrieving object $o_2$, 
while robot $r_1$ just finished dropping off an object and is ready to start to 
retrieve the next one. Since $O_{\text{NA}} = \{o_1, o_3\}$ and both objects 
are reachable, the current state generates two new states by 
assigning $o_1$ or $o_3$ to $r_1$. Since $o_3$ is close to the exit, 
if we assign $o_3$ to $r_1$, at the next state (the one in the middle 
of Fig.~\ref{fig:search}), $r_1$ finishes retrieving $o_3$ before $r_2$ reaches the exit, 
and immediately starts to retrieve $o_1$. The search finishes as 
both robots return to the exit and all objects are removed from the workspace. 

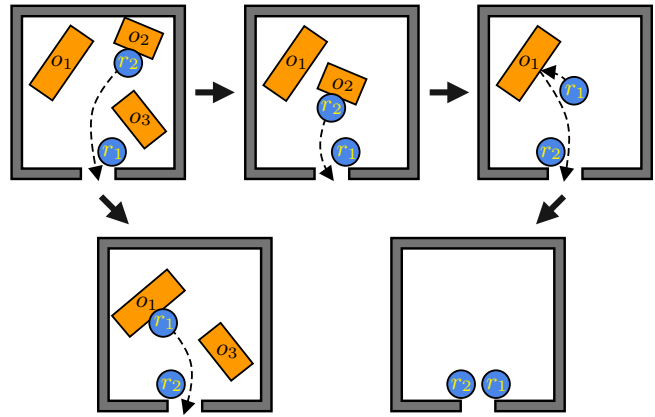
\begin{figure}[ht]
    \centering
    \small
    \begin{overpic}[width = 0.99 \linewidth]{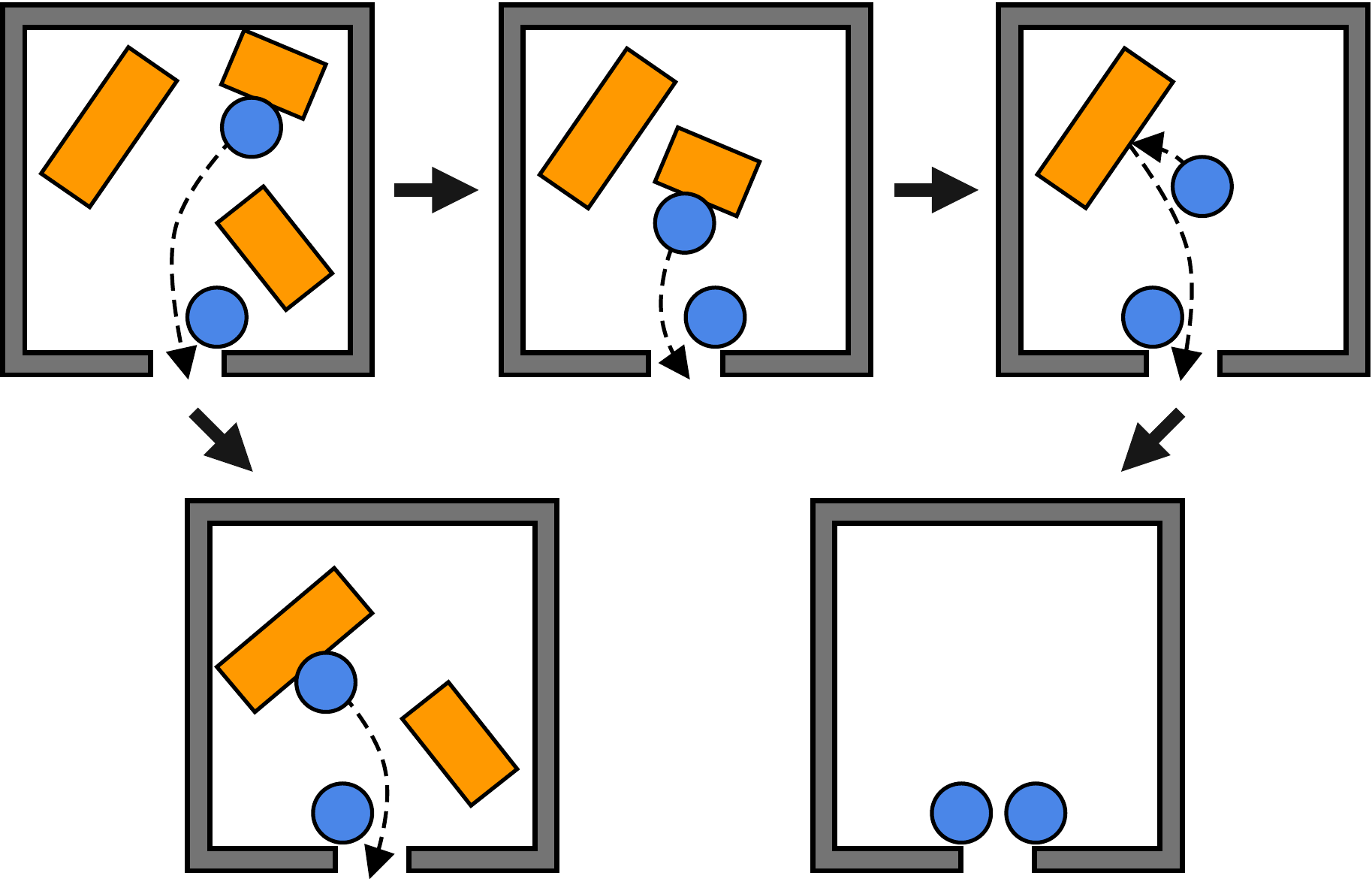}
        \put (14.3, 40.5) {\textcolor{yellow}{$r_1$}}
        \put (16.7, 54.4) {\textcolor{yellow}{$r_2$}}
        \put (6.3, 54.5) {$o_1$}
        \put (18.3, 58.2) {$o_2$}
        \put (18.5, 45.7) {$o_3$}
        \put (50.7, 40.5) {\textcolor{yellow}{$r_1$}}
        \put (48.3, 47.5) {\textcolor{yellow}{$r_2$}}
        \put (42.8, 54.5) {$o_1$}
        \put (50, 51) {$o_2$}
        \put (86.2, 50) {\textcolor{yellow}{$r_1$}}
        \put (82.4, 40.4) {\textcolor{yellow}{$r_2$}}
        \put (79, 54.5) {$o_1$}
        \put (22.3, 14) {\textcolor{yellow}{$r_1$}}
        \put (23.4, 4.3) {\textcolor{yellow}{$r_2$}}
        \put (19.5, 17) {$o_1$}
        \put (32, 9.5) {$o_3$}
        \put (74, 4.3) {\textcolor{yellow}{$r_1$}}
        \put (68.5, 4.3) {\textcolor{yellow}{$r_2$}}
    \end{overpic}
    \caption{\label{fig:search}
    Illustration of a discrete search tree section. 
    There are totally five discrete search states in this figure. 
    }
\end{figure}

When performing an object assignment during the search process, 
not all assignments seem to provide feasible low-level grasp and 
motion planning problems due to two reasons. 
First, some objects might be surrounded by the others, thus the 
grasp planner cannot find a collision-free grasping configuration. 
Second, even with a valid grasping configuration, an object might 
be behind some other ones and is not reachable. That is, there does not exist 
a collision-free path from the robot's current configuration to the 
designated object. 
As discussed in Section~\ref{subsec:structure}, simply ignoring these 
seemingly infeasible assignments could 
significantly increase the solution makespan. 
In this work, we use a {\em lookahead} method to accurately determine  
whether an assignment is feasible.

\textbf{With Lookahead.} 
To address the optimality-impacting interactions among multiple robots
and objects, at each state, lookahead is performed via simulating
the execution of robot actions for robots that are already assigned.
For example, for the case from Fig.~\ref{fig:rordep}, object $o_1$ is assigned to robot $r_1$. 
As the assignment to $r_2$ is being 
decided, we also consider a future setting where $o_1$ is already picked
up by $r_1$ and thus no longer blocking $o_2$. 
Similar procedure can be carried out for $k$ robots in a recursive manner: 
as an assignment is being decided, we virtually remove the objects 
that are immediately reachable, and then see if such a removal makes the 
previously unreachable objects reachable. 
As we will show in Section~\ref{sec:evaluation}, algorithms with  
lookahead always provide more optimal solutions than the ones 
without lookahead. 

Based on the state discretization and the search structure, we propose 
several algorithms to solve the problem near-optimally, which are 
introduced in the next few sub-sections. 
Note that we describe the algorithms as near-optimal due to the near-optimality 
of the motion planner. 

\subsection{Near-Optimal A${}^*$ Search with an Admissible Heuristic}
With the definition of the system state and the established tree structure, 
the A${}^*$ framework can be applied to find a 
near-optimal solution. For estimating the cost-to-goal, we propose 
an admissible heuristic which calculates an underestimated makespan, 
i.e., the minimum distance that the robots must travel to retrieve all 
remaining objects over the product of the number of robots $k$ and 
the maximum speed of a robot. 
Given a search state $s$, 
we denote $o_j$ as the object robot $r_j$ is currently retrieving. 
The heuristic value for $s$ is calculated as 
\[
H(s) = \frac{\sum_{1 \leq j \leq k} d(o_j) + \sum_{o_i \in O_{\text{NA}}} d(o_i)}{ kv_{\text{max}}}, 
\]
where $v_{\text{max}}$ denotes the maximum speed of the robots, 
and the function $d$ returns the shortest distance for a robot to 
retrieve an object. For the former part of the numerator, $d(o_j)$ 
calculates the remaining straight line distance for robot $r_j$ to finish retrieving 
$o_j$; for the latter part, $d(o_i)$ is two times the Euclidean 
distance between $o_i$ and the drop-off location. 
Note that since the heuristic ignores robot interaction and acceleration, 
its value never exceeds the true cost-to-goal. 
In a real system with robot's specifications available, the heuristic may 
use more accurate metrics rather than distance over maximum speed, 
to account for robot dynamics. 

As shown in Section~\ref{sec:evaluation}, the A* algorithm 
generates high-quality solutions for all test cases. However, this 
method only scales up to around eleven robots. To improve the 
scalability while maintain a high level of solution optimality, 
we propose an approximate dynamic programming (DP) algorithm. 

\subsection{Approximate Dynamic Programming} 
The dynamic programming (DP) recursion is based on the assumption that 
the near-optimal solution of the entire problem can be built on 
top of the near-optimality of its sub-problems. A sub-problem of \mrcr 
is to simply retrieve a subset of objects $O_{\text{sub}} \subseteq \O$, 
while ignoring the other objects for both task completion and collision check. 
Let $C(O_{\text{sub}})$ denote the near-optimal makespan of clearing 
all objects in $O_{\text{sub}}$, the DP framework for calculating a 
near-optimal makespan for the entire problem is demonstrated in Alg.~\ref{alg:dp}. 

\begin{algorithm}
    \DontPrintSemicolon
    $C = \{\varnothing : 0\}$\;
    \For{$1 \leq m \leq n$}{ \label{alg:dp:iteration}
        \For{$O_{\text{sub}} \in$ {\normalfont all $m$-subsets of $\O$}}{\label{alg:dp:combinations}
            $C(O_{\text{sub}}) = \displaystyle \min_{o_i \in O_{\text{sub}}}\{\min_{1 \leq j \leq k}\{C(O_{\text{sub}} \setminus \{o_i\}) + c_{ij}\}\}$\;\label{alg:dp:recursion}
        }
    }
    \Return $C(\O)$\;
    \caption{Using DP to get a near-optimal makespan.}\label{alg:dp}
\end{algorithm}

We start from the base case where no objects need to be picked up and 
gradually increase the number of objects $m$ in $O_{\text{sub}}$ (line~\ref{alg:dp:iteration}). 
For each $1 \leq m \leq n$, we iterate through all possible $O_{\text{sub}}$ 
(line~\ref{alg:dp:combinations}) and use the recursive function in 
line~\ref{alg:dp:recursion} to calculate a near-optimal makespan to 
remove all objects in this object subset. 
In the recursive function, $c_{ij}$ is the additional cost of using robot 
$r_j$ to remove object $o_i$ as compared to the makespan of 
removing all objects in $O_{\text{sub}} \setminus \{o_i\}$. Here, 
$C(O_{\text{sub}} \setminus \{o_i\})$ is already calculated in the $(m - 1)$-th 
iteration, and $c_{ij}$ can be computed by calling the multi-robot motion planner. 
It is straightforward that with proper bookkeeping, such a DP 
structure also provides the near-optimal task and motion plan associated with the 
near-optimal makespan.

The time complexity of the DP approach is as follows. 
Suppose $|O_{\text{sub}}| = m$, there are $\binom{n}{m}$ possible 
$O_{\text{sub}}$ and for each, computing the recursion requires a cost 
of $O(kn)$. This yields a total computational cost of  
\[
O(kn) \big[\binom{n}{0} + \binom{n}{1} + \dots + \binom{n}{n-1}\big]= O(kn2^n),
\]
which grows much slower than the naive $O(n!k^n)$. 

\subsection{Greedy Best-First Search}
To further improve the scalability of the discrete search-based method, 
we have developed a greedy best-first search algorithm to reduce the number of 
state visited in the entire search process. 
During the object assignment stage, for each robot, the greedy algorithm always 
selects the closest available object as its next target object. 

\subsection{Monte Carlo Tree Search} 
Trying to find a balance between scalability and optimality, 
we also implemented a search algorithm based on {\em Monte Carlo tree search} (MCTS) 
\cite{coulom2006efficient, kocsis2006bandit}. 
Here in this sub-section, we first provide a brief review of the MCTS framework 
and then focus on explaining how we adapt it for \mrcr. 

MCTS is a search algorithm that expands the search tree based on the analysis of 
the most promising states. The method's essential components are 
{\em selection}, {\em expansion}, {\em simulation}, and {\em back-propagation}. 
In a basic iteration of MCTS, first, the selection stage picks a search node based 
on a selection function. Then, the expansion stage creates child nodes of 
the selected node. After that, the {\em reward values} of the new child nodes are determined 
by a simulation from the node to an end state. Finally, the back-propagation stage 
updates the tree to prepare for the next selection stage. 

In our implementation, the selection phase uses an Upper Confidence Bound (UCB) 
formula $\bar{T} + \sqrt{\frac{\ln{N}}{n}}$ to select the next node to explore. 
Here, $\bar{T}$ is the average simulated child node makespan, $n$ is the number of times the 
node is expanded, and $N$ is the number of parent expansions. 
As a node is expanded, we visit all child nodes and estimate their makespan 
as follows: 
for high-level decision making, we use the greedy best-first strategy; 
for low-level execution time prediction, instead of calling the motion planner, 
we simply use the picking distance over robot speed, and ignore interactions between robots.
The reward of each simulation is back-propagated to inform further expansion.
Every time after we expanded three levels of the search tree, the most explored node is chosen as the next action. 


\section{Experimental Studies}\label{sec:evaluation}
We evaluate the performance of the proposed near-optimal and heuristic 
search algorithms under varied object density and drop-off location setups. 
The algorithms were implemented in C++ and executed on a quad-core Intel 
CPU at 3.3GHz with 32GB RAM. A video of simulated Kuka youBots carrying 
out the tasks is provided that supplements the evaluation described in 
this section. 

The testing environments we used for evaluation can be categorized 
into two general types: one with objects sampled close to each other 
in a {\em clutter}, and the other one with objects {\em scattered} inside the 
workspace. An illustration of these two typical environments are 
shown in Fig.~\ref{fig:test_env}. For evaluation, we use disc robots 
with identical radius; some randomly selected solutions are then 
executed in the V-REP simulator~\cite{RohSin2013IROS} with an accurate
Kuka youBot model, and the execution results are consistent with 
the evaluation. We assume that the picking and the unloading time 
are the same and equal the time it takes for a robot to travel 
half of the square workspace side length at maximum speed. This 
adjustable picking/unloading time is included in the total execution 
time; our choice is based on our actual experience working with 
Kuka youBots.  

For multi-robot motion planning, in the evaluation, a shortest path 
connecting a robot to the assigned object is computed using visibility 
graph~\cite{lozano1979algorithm} which treats objects as static 
obstacles but does not consider interactions among the robots. Then, 
to generate feasible robot trajectory, collisions among robots are 
resolved using RVO~\cite{van2011reciprocal}.   
\begin{figure}[ht!]
    \centering
    \includegraphics[width = 0.88\linewidth]{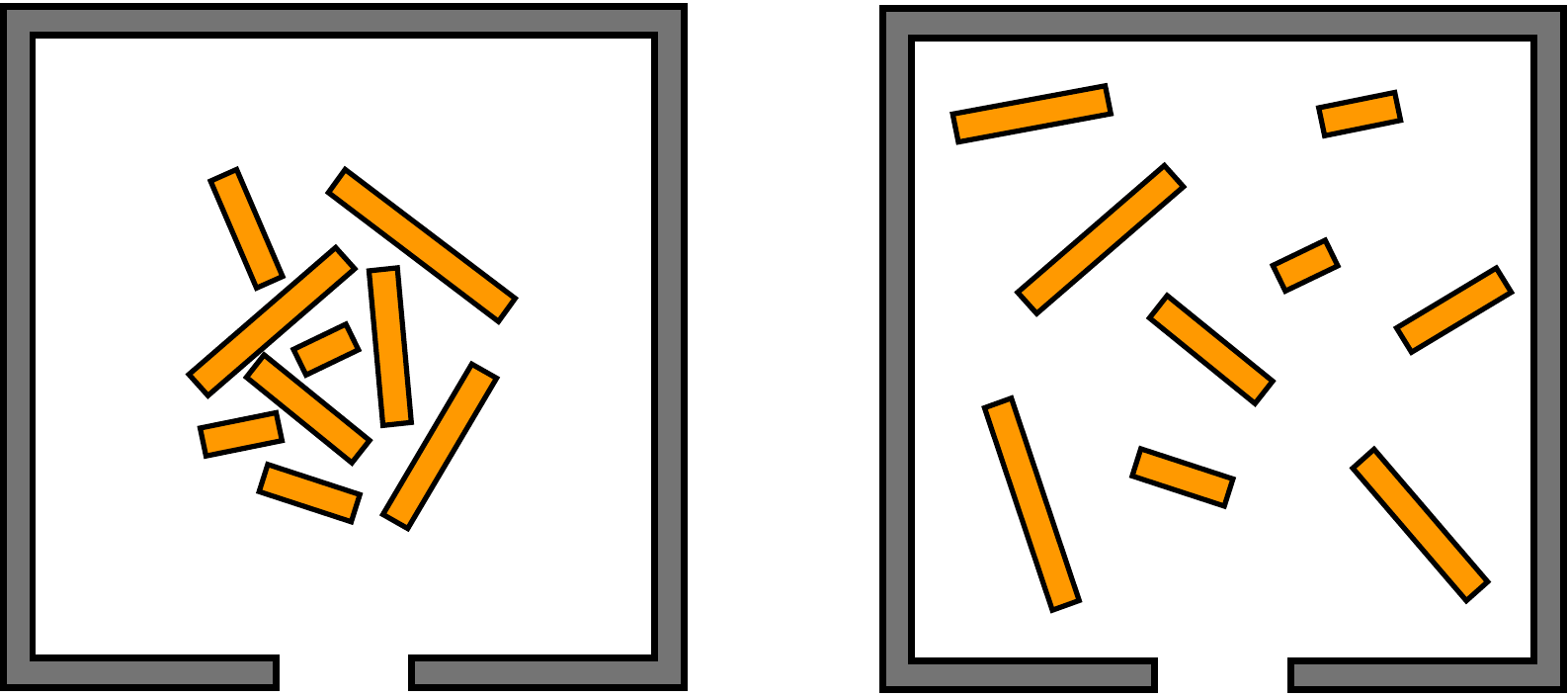}
    \caption{\label{fig:test_env}
    Two different categories of test environments. [left] A cluttered 
		setting. [right] Objects are more dispersed in the environment.  
    }
\end{figure}

Our first set of experiments evaluates the performance of all 
proposed algorithms using both cluttered and scattered setups 
(Fig.~\ref{fig:test_env}), with the number of robots $k = 2$ 
and the number of objects $n$ ranging from $5$ to $15$. 
Fig.~\ref{fig:two-robots} shows the result on optimality ratio and 
computation time. The optimality ratio is the solution makespan of 
the proposed algorithms divided by the optimal makespan of {\em single} 
robot clutter removal for the same setup. The results show that for 
all test cases, the near-optimal (A${}^*$ and DP based) algorithms 
achieve an optimality ratio as low as $0.6$, which is close to 
the underestimated theoretical lower bound $0.5$. As the number of objects 
increases, the difference in optimality ratio between near-optimal 
algorithms and heuristic search algorithms can become as large 
as $5\%$. In terms of scalability, the A${}^*$ search method 
scales up to about $11$ objects with the computation time limited 
under $400$ seconds, while DP can handle $15$ objects. The heuristic 
search methods can solve a problem with $15$ objects in under $30$ 
seconds. 

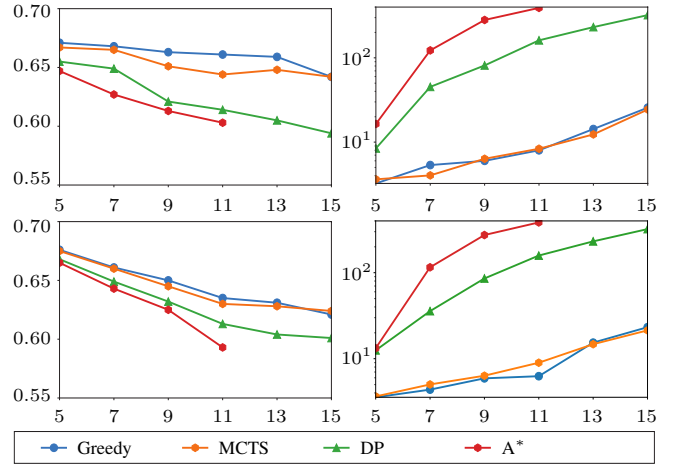
\begin{figure}[ht!]
    \begin{overpic}[width=0.99\columnwidth, grid = 5]{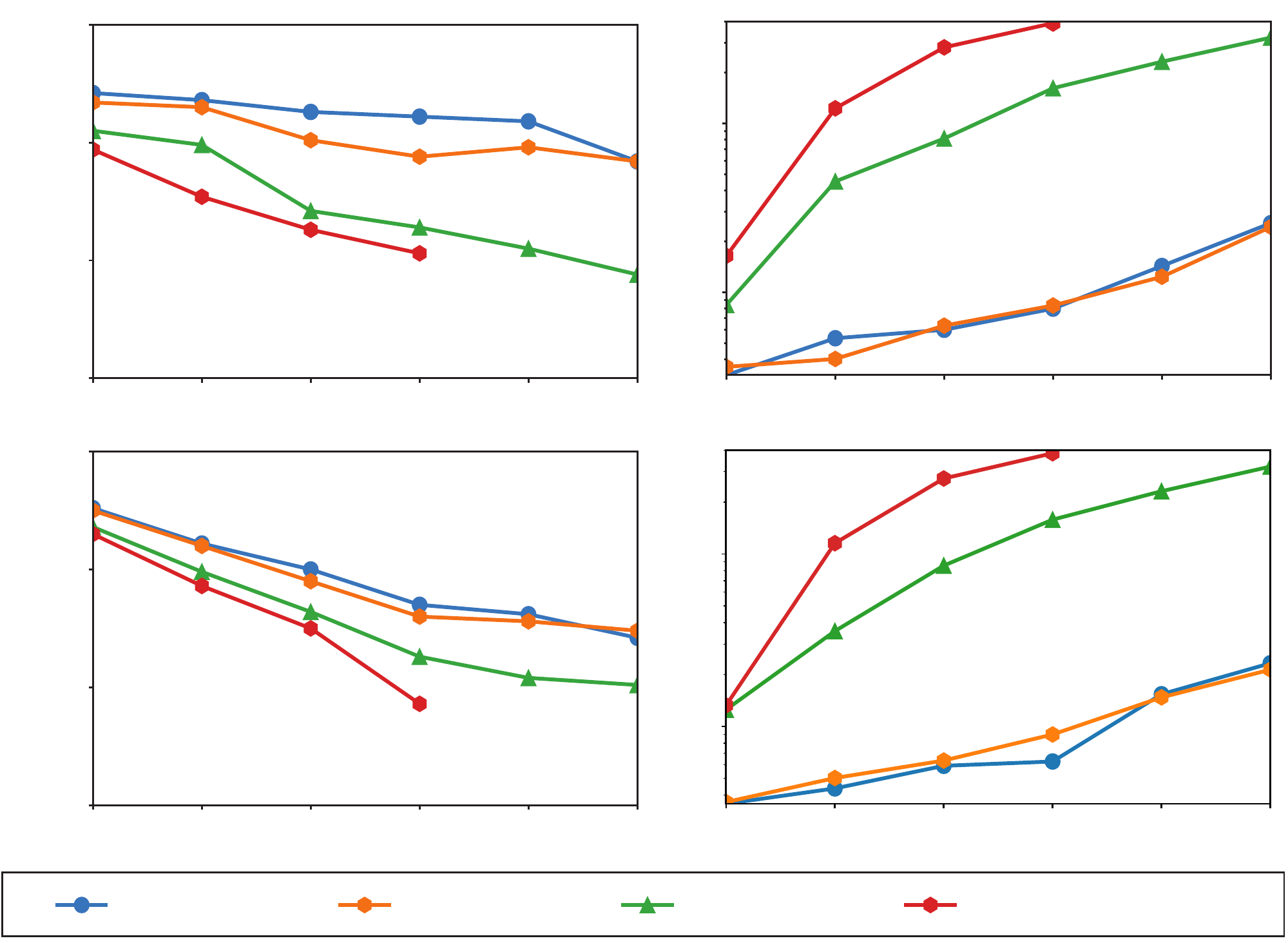}
        \scriptsize
        \put(0, 70) {$0.70$}
        \put(0, 61.5) {$0.65$}
        \put(0, 52.8) {$0.60$}
        \put(0, 44) {$0.55$}
        \put(0, 36.5) {$0.70$}
        \put(0, 27.9) {$0.65$}
        \put(0, 19) {$0.60$}
        \put(0, 10.2) {$0.55$}
        \put(51, 16) {$10^1$}
        \put(51, 29) {$10^2$}
        \put(51, 49.5) {$10^1$}
        \put(51, 62.5) {$10^2$}
        \put(6.4,  39.8) {$5$}
        \put(14.9, 39.8) {$7$}
        \put(23.3, 39.8) {$9$}
        \put(31  , 39.8) {$11$}
        \put(39.3, 39.8) {$13$}
        \put(47.8, 39.8) {$15$}
        \put(55.6, 39.8) {$5$}
        \put(64  , 39.8) {$7$}
        \put(72.5, 39.8) {$9$}
        \put(80.2, 39.8) {$11$}
        \put(88.4, 39.8) {$13$}
        \put(96.9, 39.8) {$15$}
        \put(6.4,  7) {$5$}
        \put(14.9, 7) {$7$}
        \put(23.3, 7) {$9$}
        \put(31  , 7) {$11$}
        \put(39.3, 7) {$13$}
        \put(47.8, 7) {$15$}
        \put(55.6, 7) {$5$}
        \put(64  , 7) {$7$}
        \put(72.5, 7) {$9$}
        \put(80.2, 7) {$11$}
        \put(88.4, 7) {$13$}
        \put(96.9, 7) {$15$}
        \put(10,  1.9) {Greedy}
        \put(32,  1.9) {MCTS}
        \put(54,  1.9) {DP}
        \put(76,  1.9) {A${}^*$}
	\end{overpic}
    \caption{\label{fig:two-robots}
    Optimality ratio and computation time comparison of all the removal 
		sequence search methods for \mrcr with two robots. The $x$ axis on all 
		figures are the number of objects in a given environment. The $y$ axes 
		for the figures on the left refer to optimality ratio as compared with
		the single-robot case. The $y$ axes for the figures on the right are 
		the computation time in seconds. [top] Settings where objects are more 
		cluttered. [bottom] Settings where objects are more scattered. 
    }
		\vspace*{-2mm}
\end{figure}
With a closer look, we can observe that, when the cluttered case is 
compared with the scattered case, A${}^*$ and DP generally does 
slightly better in the former, except at one or two outlier points. 
On the other hand, the greedier methods show an opposite trend and 
do better in the scattered case. 
By examining the object picking sequences generated in these settings, 
it appears that the difference can be explained as follows. From the 
earlier study of \crp \cite{TanYu2019ISRR}, we know that greedy methods 
work well for a single robot, producing solutions generally 
indistinguishable from the optimal. As we move to two robots, for the 
scattered case, greedy methods are likely to cause the robots to pick 
objects in different part of the workspace, essentially running a 
single robot greedy solution with some minor coordination. 
On the other hand, since the cluttered case is harder than the scattered one, 
there are generally more opportunities for multiple robots to optimize the picking sequence. 
However, greedy methods will not exploit such optimization as much given its short 
horizon. Due to the closeness of the objects to be removed, the greedy method actually 
spend more time coordinating the robots' motion.

In the second set of experiments, we evaluate the heuristic search 
algorithms (i.e., greedy search and MCTS) in the scattered environment as 
the number of robots ranging from $1$ to $5$ and the number of objects
ranging from $10$ to $25$. Since the number of robots is increased, the 
exit width is also increased so that it potentially allows two robots to unload 
at the same time, to avoid making the exit an artificial bottleneck. The 
experimental result is summarized in Fig.~\ref{fig:two-five}. 
As we can observe, first, as expected, with a larger exit, the two-robot case 
does better than the previous experiment. 
Second, whereas adding more robots shortens the makespan, the amount of 
gain is quickly diminishing (it does so even faster if we use a narrower
exit, which we have attempted). Nevertheless, with five robots, the 
greedy methods can readily handle $25$ objects and does so about $2.5$ times 
faster than if a single robot is used. We note that the setup is chosen to 
be somewhat constrained; if multiple exits are available, additional 
execution time speedup can be expected. 


\begin{figure}[ht!]
		\vspace*{-2mm}
    \includegraphics[width=\linewidth]{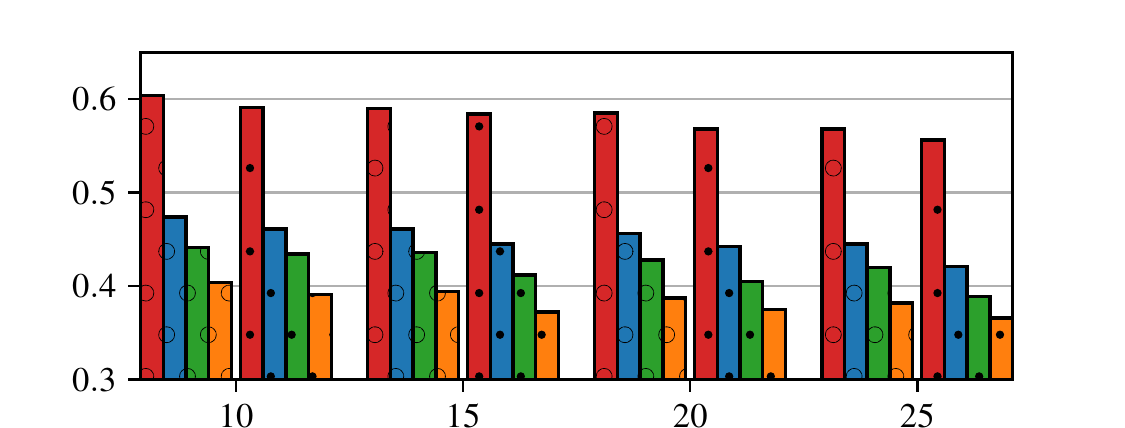}
    \includegraphics[width=\linewidth]{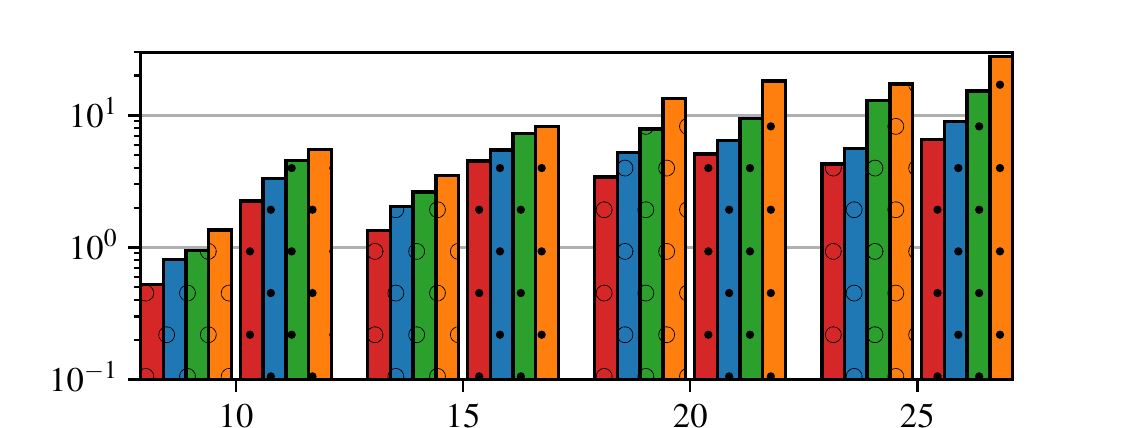}
		\vspace*{-4mm}
    \caption{\label{fig:two-five}
		Performance of greedy and MCTS methods for $2$-$5$ robots and $10$-$25$ objects.
		The $x$ axes for both figures denote the number of objects. The top figure 
		is the makespan as compared with a single robot case. The bottom figure are 
		the computation time in seconds. For each number of objects, there are 
		eight bars. The left (resp., right) four bars correspond to greedy search
		(resp., MCTS search) for $2, 3, 4, 5$ robots, from left to right.
    }
\end{figure}

In a third set of experiments presented here, we work with a cluttered setting
to evaluate the impact of the complex multi-robot multi-object dependency in \mrcr, 
as explained in Section~\ref{subsec:structure}. 
In the experiment, we run each of the four search methods 
with lookahead (Section~\ref{sec:algorithm}) enabled or disabled (by default 
lookahead is enabled). The result is plotted in Fig.~\ref{fig:lh}. The 
solid lines are the same as that from the first plot in 
Fig.~\ref{fig:two-robots}. As expected, the interaction 
contributes positively to optimality, yielding a gain up to about $5\%$. 

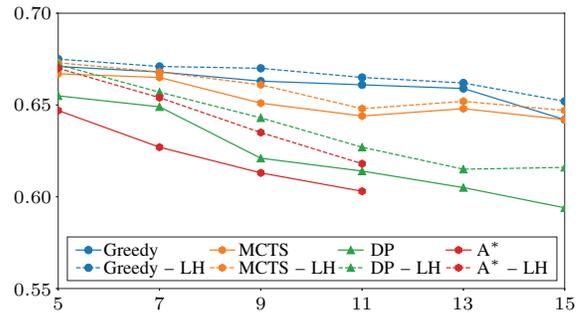
\begin{figure}[ht!]
\vspace*{0mm}
    \begin{overpic}[width=0.99\columnwidth]{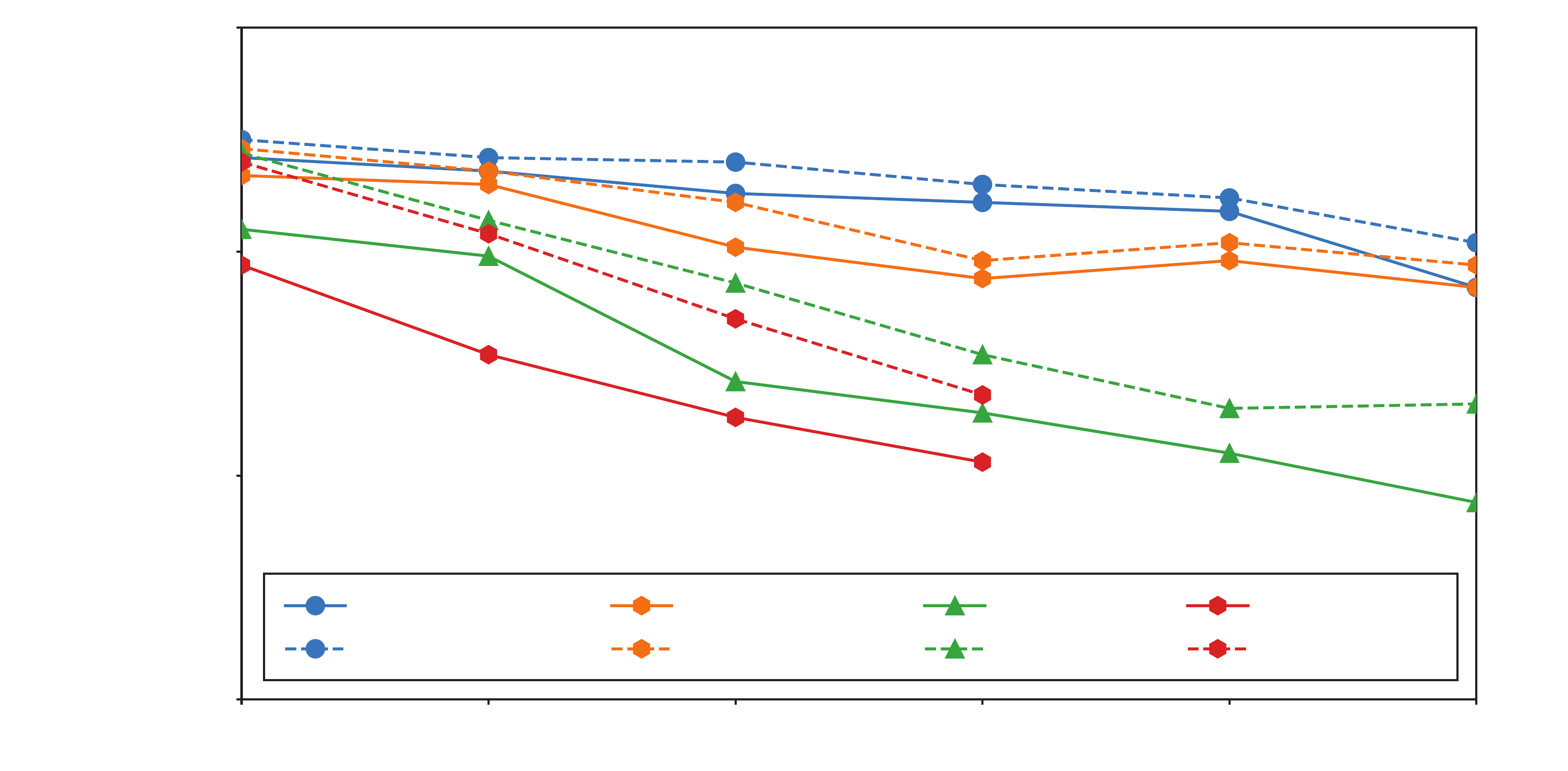}
        \scriptsize
        \put(8.5, 46.5) {$0.70$}
        \put(8.5, 32.2) {$0.65$}
        \put(8.5, 17.9) {$0.60$}
        \put(8.5,  3.5) {$0.55$}
        \put(14.6, 1.2) {$5$}
        \put(30.3, 1.2) {$7$}
        \put(46.1, 1.2) {$9$}
        \put(61,   1.2) {$11$}
        \put(76.7, 1.2) {$13$}
        \put(92.6, 1.2) {$15$}
        \put(22.5, 9.6) {Greedy}
        \put(22.5, 6.9) {Greedy -- LH}
        \put(43.5,   9.6) {MCTS}
        \put(43.5,   6.9) {MCTS -- LH}
        \put(64, 9.6) {DP}
        \put(64, 6.9) {DP -- LH}
        \put(80.5, 9.6) {A${}^*$}
        \put(80.5, 6.9) {A${}^*$ -- LH}
    \end{overpic}
    \caption{\label{fig:lh}
    Comparing methods with and without performing look-ahead to address 
        the multi-robot multi-object dependency.
    The $x$ axis is the number of objects in a given environment. 
    The $y$ axis is the optimality ratio as compared with the single-robot case.
		``method -- LH'' means a given method without look-ahead. 
    }
\end{figure}

Lastly, as mentioned, in V-REP based simulation with accurate Kuka youBot 
models, the computed object removal sequence and the associated robot 
assignments returned from our task planner directly carry over; the amount 
of execution time that is saved is also largely the same as those shown in 
Fig.~\ref{fig:two-robots} and Fig.~\ref{fig:two-five}. This further 
validates the effectiveness of our pipeline design. Selected V-REP simulations 
can be found in the accompanying video.

\section{Conclusion and Discussions}\label{sec:conclusion}
In this work, we have explored the combinatorial challenge present in 
the task and motion planning problem of removing clutter from an 
environment with limited ingress/egress access using multiple robots. 
We call the formulation the \mrcr problem. In contrast to the single 
robot case \cite{TanYu2019ISRR}, for multiple robots, in addition to 
having a much larger search space due to the choice of more robots, 
unique constraints also arise that make (near-) optimal task planning 
more computationally demanding. Toward addressing these combinatorial 
challenges in \mrcr, we have proposed an extendable solution pipeline 
and within it multiple principled search algorithms (greedy, MCTS, DP, 
and A${}^*$) that balance between scalability and solution 
optimality. In general, however, these search methods are all 
reasonably practical and produce significant savings in task execution 
time when the solutions are compared with those from the optimal single 
robot setting. For example, using two robots, the execution time can 
be as little as less than $60\%$ of the optimal single robot case, 
approaching the theoretical lower bound of $50\%$. Moreover, when 
integrated into physics based simulation in realistic settings, the 
performance of our algorithms hold unchanged. Strengthening the 
result, we also show that computing the optimal object removal 
sequence remains NP-hard for multiple robots. 
In conclusion, we have developed several effective search algorithms 
for computing high-quality object removal sequence for solving \mrcr.

\bibliographystyle{formatting/IEEEtran}
\bibliography{bib/all}

\end{document}